\documentclass{article}

\usepackage[preprint]{neurips_2026}

\usepackage[utf8]{inputenc} 
\usepackage[T1]{fontenc}    
\usepackage{booktabs}       
\usepackage{amsfonts}       
\usepackage{nicefrac}       
\usepackage{microtype}      
\usepackage{xcolor}
\usepackage[
  pagebackref=true,
  breaklinks=true,
  colorlinks=true,
  urlcolor={red!90!black},
  linkcolor={red!90!black},
  citecolor={blue!90!black},
  bookmarks=false
]{hyperref}
\usepackage{graphicx}
\usepackage{amsthm}
\usepackage{mathtools}
\usepackage{algorithm}
\usepackage{algpseudocode}
\usepackage{multirow}
\usepackage[table]{xcolor}
\usepackage{subcaption}
\usepackage{placeins}
\usepackage{float}

\usepackage{amsmath,amsfonts,bm}

\def\rvc{{\boldsymbol{c}}}

\def\rvx{{\boldsymbol{x}}}

\newcommand{\norm}[1]{\left\lVert#1\right\rVert}
\newtheorem{prop}{Proposition}

\newtheorem{cor}{Corollary}

\theoremstyle{remark}

\usepackage{enumitem}
\usepackage{hyperref}       
\usepackage{thmtools}
\usepackage{thm-restate}

\definecolor{todopurple}{HTML}{9370DB}

\title{Beyond Pairwise Preferences: Listwise Reward-Aware Alignment for Diffusion Models}

\author{%
  Austin Wang \\
  Caltech \\
\texttt{akwang@caltech.edu} \\
  \And
  Jiaqi Han \\
  Stanford University \\
\texttt{jiaqihan@stanford.edu} \\
  \AND
  Stefano Ermon \\
  Stanford University \\
\texttt{ermon@cs.stanford.edu} \\
  \And
  Yisong Yue \\
  Caltech \\
  \texttt{yyue@caltech.edu} \\
}

\begin{document}

\maketitle

\begin{abstract}

  Preference optimization has emerged as an efficient alternative to online reinforcement learning from human feedback (RLHF) for aligning text-to-image diffusion models. However, existing methods largely reduce supervision to binary pairwise comparisons. This pairwise reduction is limiting when training data naturally contains multiple candidate images for the same prompt, and when continuous reward scores can provide richer information than a single winner–loser label. To address these limitations, we propose Diffusion LAIR, a reward-aware listwise preference optimization method for diffusion models. For each prompt, LAIR converts reward scores across a group of candidate images into centered advantage weights, then optimizes an advantage-weighted regression objective on the implicit reward, defined as the denoising-loss improvement of the current model over a fixed reference model, with a quadratic penalty that regularizes the magnitude of the implicit reward. The resulting objective uses all candidates simultaneously rather than selecting pairs, and remains conservative by explicitly controlling the magnitude of the implicit reward. The LAIR objective admits a bounded closed-form optimum in implicit-reward space, clarifying how the regularization strength controls the magnitude of the preference update. Experiments show that Diffusion LAIR outperforms strong  preference optimization baselines on SD1.5 and SDXL across text-to-image generation, compositional generation, and image editing benchmarks. Model checkpoints can be found \href{https://huggingface.co/collections/austin-k-wang/diffusion-lair-models}{here}.
\end{abstract}

\section{Introduction}

\begin{figure}[!h]
    \centering
    \includegraphics[
        width=0.9\linewidth,
        height=0.9\textheight,
        keepaspectratio
    ]{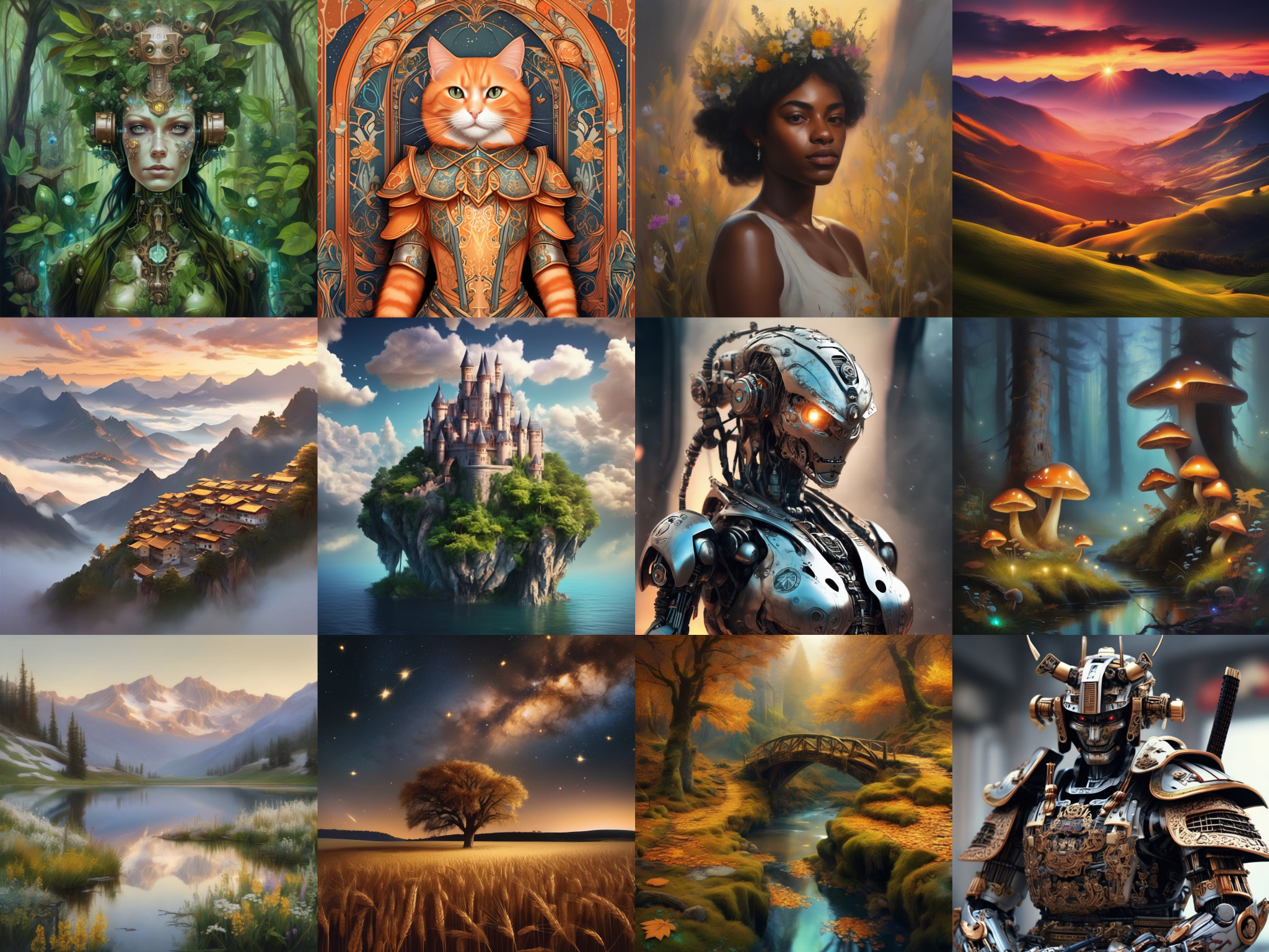}
    \caption{Sample images generated from SDXL trained with Diffusion LAIR.}
    \label{fig:sdxl-compare-baselines}
    \vspace{-0.1in}
\end{figure}


Text-to-image diffusion models, such as Stable Diffusion \citep{podell2023sdxl, rombach2022high} and Imagen \citep{saharia2022photorealistic}, have become a standard backbone for high-quality image generation. Large-scale pretraining, however, does not by itself guarantee alignment with human preferences. Post-training is therefore needed to improve aspects like aesthetics, prompt following, and overall sample quality. As a result, diffusion preference optimization has emerged as a practical approach to this problem by adapting preference-learning objectives to denoising diffusion models \citep{wallace2024diffusion, zhu2025dspo, li2024aligning, hong2026margin, han2025discrete, li2025divergence}.



Although online reinforcement learning algorithms for diffusion alignment \citep{fan2023dpok, black2023training, clark2023directly, liu2025flow, zheng2025diffusionnft, choi2026rethinking} have been extensively explored, their on-policy nature tightly couples generation, reward evaluation, and policy updates, making training expensive and difficult to parallelize. Offline preference optimization avoids this cost by learning from pre-collected preference data. This setting is increasingly natural: modern preference pipelines often generate multiple candidate images for the same prompt and evaluate them with reward models that provide graded, rather than purely binary, supervision \citep{lee2025calibrated, karthik2025scalable}. Yet most diffusion preference objectives still inherit the pairwise structure of direct preference optimization (DPO), where each update compares one preferred image against one dispreferred image. This pairwise reduction discards the remaining candidates and ignores the magnitude of reward gaps. Recent approaches \citep{lee2025calibrated, liang2025aesthetic} use vision-language reward models to construct richer preference signals, but still rely on pair selection heuristics to develop pairwise objectives.

To address this limitation, we propose Diffusion LAIR (Listwise Advantage-weighted Implicit Reward), which treats diffusion preference optimization as a listwise, reward-aware learning problem. For each prompt, Diffusion LAIR converts reward scores across a group of candidate images into centered advantage weights, which promote high-reward samples and suppress low-reward samples relative to the group. Diffusion LAIR then optimizes an advantage-weighted implicit reward, defined by the denoising-loss improvement of the current model over a fixed reference model, with a quadratic penalty that keeps the update conservative. This construction preserves the reference-based structure of diffusion preference optimization while distributing learning signal across all candidates rather than selected comparison pairs. The resulting objective admits a closed-form optimum in implicit-reward space, where each sample receives a finite target proportional to its centered advantage weight, clarifying how regularization controls the size of the preference update.

\textbf{Our contributions}
\vspace{-0.5pt}
\begin{itemize}[leftmargin=1.5em,itemsep=0.05em, topsep=0pt]
    \item We propose Diffusion LAIR, a preference optimization objective that learns directly from groups of reward-scored images, rather than reducing supervision to binary preference pairs.
    \item We show that the Diffusion LAIR objective admits a bounded closed-form optimum in implicit-reward space, providing insight into how regularization controls preference updates and induced distribution shift.
    \item We demonstrate strong empirical performance across text-to-image generation, compositional generation, and instruction-based image editing on SD1.5 and SDXL.
\end{itemize}


\section{Related Work}

\textbf{Reward and Preference Alignment.}
Online reinforcement learning post-training is widely used to align foundation models with human preferences, particularly for language models~\citep{christiano2017deep,stiennon2020learning,ziegler2019fine,bai2022training}. For diffusion models, prior work either backpropagates differentiable rewards through denoising chains~\citep{clark2023directly,prabhudesai2023aligning,ren2025half} or applies RL objectives to sampled diffusion trajectories~\citep{black2023training,fan2023dpok,liu2025flow,xue2025dancegrpo,zheng2025diffusionnft,xue2025advantage,choi2026rethinking,han2025discrete,ye2025data}. These approaches can achieve strong reward gains but often require expensive trajectory sampling and may be sensitive to reward hacking or optimization instability. Complementary inference-time methods steer generation through reward-guided sampling, search, or derivative-free posterior inference, but leave model parameters unchanged and incur additional sampling cost~\citep{singhal2025general,uehara2025inference,bansal2024universal,yeh2025training,zheng2025blade,wang2025ensemble}. In contrast, we study offline preference alignment, where the diffusion model is post-trained solely from pre-collected preference data.

\textbf{Offline Preference Alignment for Diffusion Models.}
Offline preference optimization has become an attractive alternative to traditional RLHF techniques. Crucially, these methods only require preference data, rather than full denoising trajectories, making them significantly more efficient than online RL algorithms. \citep{wallace2024diffusion, yang2024using} are pioneer works that extend DPO \citep{rafailov2023direct} to diffusion models, learning directly from human preference data. More recent state-of-the-art works include DSPO \citep{zhu2025dspo}, which integrates DPO with denoising score matching; Diffusion KTO \citep{li2024aligning}, which formulates alignment as utility maximization from per-sample binary feedback; and InPO \citep{lu2025inpo}, which improves efficiency by using reparameterized DDIM inversion to identify and optimize latent variables that are most relevant to preference alignment. Nevertheless, most offline diffusion alignment methods ultimately reduce supervision to pairwise comparisons with binary preference labels, missing out on finer-grained signals like listwise rankings and reward scores. This pairwise--listwise distinction also connects our setting to classical learning-to-rank methods such as RankNet, ListNet, and ListMLE~\citep{burges2005ranknet,cao2007listnet,xia2008listmle}, which similarly distinguish pairwise and listwise supervision, though our objective operates on diffusion implicit rewards rather than learned ranking scores.

Recent methods incorporate reward signals to construct rankings or calibrate pairwise preferences~\citep{karthik2025scalable,lee2025calibrated}, but still optimize pairwise or ordinal objectives. Most closely related, Diffusion LPO~\citep{bai2025towards} preserves list-level structure using a Plackett--Luce objective over ranked image lists, but primarily exploits ordinal supervision. In contrast, Diffusion LAIR directly uses continuous reward scores over the full candidate list, preserving both relative ordering and cardinal reward gaps through candidate-wise implicit-reward targets.

\section{Background and Preliminaries}

\textbf{Diffusion Models.}
Denoising diffusion models \citep{ho2020denoising, song2020score, karras2022elucidating} are generative models that enable sampling from a data distribution $q(\rvx_0)$. Given a noise schedule $\alpha_t$ and $\sigma_t$, they define a reverse Markov process $p_\theta(\rvx_{0:T}) = p_\theta(\rvx_T) \prod_{t=1}^T p_\theta(\rvx_{t-1} \mid \rvx_t),$
where each reverse transition is typically parameterized as a Gaussian, $p_\theta(\rvx_{t-1} \mid \rvx_t)
= \mathcal{N}\!\left(\rvx_{t-1}; \mu_\theta(\rvx_t, t), \Sigma_t \right).$
Training is commonly carried out by minimizing the evidence lower bound (ELBO), which in practice yields the standard denoising objective
\[
\mathcal{L}_{\mathrm{DM}}
=
\mathbb{E}_{\rvx_0,\epsilon,t,\rvx_t}
\left[
\omega(\lambda_t)\,\|\epsilon-\epsilon_\theta(\rvx_t,t)\|_2^2
\right],
\]
where $\epsilon \sim \mathcal{N}(0,I)$, $t \sim \mathcal{U}(0,T)$, and $\rvx_t \sim q(\rvx_t \mid \rvx_0)
= \mathcal{N}(\alpha_t \rvx_0, \sigma_t^2 I)$. Here, $\lambda_t = \alpha_t^2 / \sigma_t^2$ denotes the signal-to-noise ratio, and $\omega(\lambda_t)$ is a pre-specified timestep weighting function.

\textbf{Direct Preference Optimization.}
Direct Preference Optimization (DPO) \citep{rafailov2023direct} is a preference alignment method that directly optimizes a policy from pairwise preference data. In the standard KL-regularized preference alignment setting, one seeks a policy \(p_{\theta}(\rvx \mid \rvc)\) that maximizes reward while remaining close to a fixed reference model \(p_{\mathrm{ref}}(\rvx \mid \rvc)\):
\begin{equation}
    \max_{p_{\theta}}
    \;
    \mathbb{E}_{\rvc \sim \mathcal{D},\, \rvx \sim p_{\theta}(\cdot \mid \rvc)}
    \bigl[r(\rvc,\rvx)\bigr]
    -
    \beta D_{\mathrm{KL}}\!\left(p_{\theta}(\rvx \mid \rvc)\,\|\,p_{\mathrm{ref}}(\rvx \mid \rvc)\right).
\end{equation}

The corresponding optimal policy takes the reward-tilted form
$p^{*}(\rvx \mid \rvc) \propto p_{\mathrm{ref}}(\rvx \mid \rvc)
\exp\!\left(r(\rvc,\rvx)/\beta\right)$, which implies
$r(\rvc,\rvx) = \beta \log \frac{p^{*}(\rvx \mid \rvc)}
{p_{\mathrm{ref}}(\rvx \mid \rvc)} + \beta \log Z(\rvc)$, where \(Z(\rvc)\) is a partition function independent of \(\rvx\). Thus, the reward can be represented, up to an additive constant, by the log-ratio between the optimal policy and the reference policy (often referred to as the implicit reward). 

Given pairwise preference data \((\rvc,\rvx^{w},\rvx^{l})\), where \(\rvx^{w}\) is preferred over \(\rvx^{l}\), DPO combines this reparameterization with the Bradley-Terry model \citep{bradley1952rank} to obtain
\begin{equation}
    \mathcal{L}_{\mathrm{DPO}}(\theta)
    =
    -
    \mathbb{E}_{(\rvc,\rvx^{w},\rvx^{l})}
    \left[
        \log \sigma \left(
            \beta \log \frac{p_{\theta}(\rvx^{w} \mid \rvc)}{p_{\mathrm{ref}}(\rvx^{w} \mid \rvc)}
            -
            \beta \log \frac{p_{\theta}(\rvx^{l} \mid \rvc)}{p_{\mathrm{ref}}(\rvx^{l} \mid \rvc)}
        \right)
    \right].
\end{equation}

While the log-ratio between $p_\theta$ and $p_\mathrm{ref}$ is easy to compute for autoregressive language models, it is intractable for diffusion models, motivating diffusion-specific approximations based on denoising performance.


\textbf{DPO for Diffusion Models.}
Diffusion-DPO \citep{wallace2024diffusion} adapts DPO to diffusion models by replacing the intractable log-likelihood ratio with a surrogate based on denoising performance. Given a noisy latent
\(\rvx_t = \alpha_t \rvx_0 + \sigma_t \epsilon\)
at timestep \(t\) under conditioning \(\rvc\), let
\begin{equation}
    l_{\theta}(\rvx_t, t, \rvc, \epsilon)
    :=
    \norm{\epsilon_{\theta}(\rvx_t, t, \rvc) - \epsilon}_2^2,
    \qquad
    l_{\mathrm{ref}}(\rvx_t, t, \rvc, \epsilon)
    :=
    \norm{\epsilon_{\mathrm{ref}}(\rvx_t, t, \rvc) - \epsilon}_2^2
\end{equation}
denote the denoising squared errors of the current model and a fixed reference model, respectively. Let
\(\lambda_t := \alpha_t^2/\sigma_t^2\) denote the signal-to-noise ratio. Since the diffusion ELBO decomposes into timestep-dependent weighted denoising errors, the log-ratio
\(\log \frac{p_{\theta}(\rvx_0 \mid \rvc)}{p_{\mathrm{ref}}(\rvx_0 \mid \rvc)}\)
can be approximated by the expected denoising-loss improvement
\begin{equation}
    S_{\theta}(\rvx_0,\rvc)
    :=
    \mathbb{E}_{t,\epsilon}
    \left[
        s_{\theta}(\rvx_0,\rvx_t,t,\rvc,\epsilon)
    \right],
    \qquad
    s_{\theta}(\rvx_0,\rvx_t,t,\rvc,\epsilon)
    :=
    \omega(\lambda_t)
    \left(
        l_{\mathrm{ref}}
        -
        l_{\theta}
    \right),
    \label{eq:implicit-reward}
\end{equation}
where \(\omega(\lambda_t)>0\) is the timestep-dependent coefficient induced by the diffusion ELBO, often set as a constant in practice \citep{wallace2024diffusion, xue2025advantage, ho2020denoising}. Thus, \(S_\theta(\rvx_0,\rvc)\) approximates the log-ratio at the clean sample level, while \(s_\theta(\rvx_0,\rvx_t,t,\rvc,\epsilon)\) is a single Monte Carlo contribution used during training.

For pairwise preference data \((\rvc,\rvx_0^w,\rvx_0^l)\), where \(\rvx_0^w\) is preferred to \(\rvx_0^l\), the ideal Diffusion-DPO objective compares \(S_\theta(\rvx_0^w,\rvc)\) and \(S_\theta(\rvx_0^l,\rvc)\). In practice, this expectation is estimated by sampling timesteps and noise, giving
\begin{equation}
    \mathcal{L}_{\mathrm{DiffusionDPO}}(\theta)
    =
    - \mathbb{E}_{(\rvc,\rvx_0^w,\rvx_0^l),\,t,\,\epsilon^w,\,\epsilon^l}
    \left[
        \log \sigma
        \left(
            \beta
            \left(
                s_{\theta}(\rvx_0^w,\rvx_t^w,t,\rvc,\epsilon^w)
                -
                s_{\theta}(\rvx_0^l,\rvx_t^l,t,\rvc,\epsilon^l)
            \right)
        \right)
    \right],
    \label{eq:diffusion-dpo}
\end{equation}
where \(\beta>0\) is the preference temperature. Thus, Diffusion-DPO encourages preferred samples to receive larger implicit reward contributions than dispreferred samples, pushing the model to denoise preferred images better than the reference relative to dispreferred images.

\section{Method}

\begin{figure}[!h]
    \centering
    \includegraphics[
        width=0.9\linewidth]{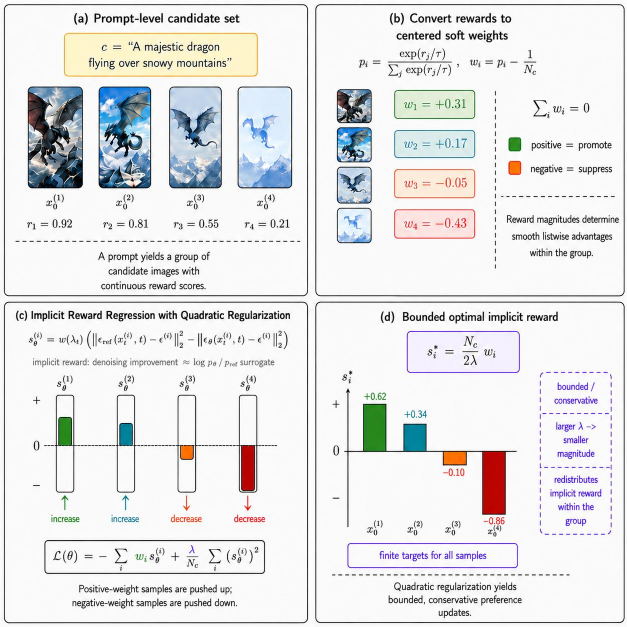}
    \caption{Overview of Diffusion LAIR. For each prompt, candidate images are scored by a reward model, converted into centered listwise advantage weights, and used to shape the diffusion implicit reward through a quadratic-regularized objective. The resulting optimum gives finite, $\lambda$-controlled implicit-reward targets that redistribute implicit reward within the candidate group.}
    \label{fig:method-diagram}
\end{figure}

\subsection{From Pairwise Preference Optimization to Listwise Reward-Aware Alignment}

Prior diffusion preference optimization methods are largely built around pairwise comparisons with binary preference labels. In contrast, we observe that many practical preference datasets, such as Pick-a-Pic~\citep{kirstain2023pick}, contain multiple images for each prompt, yielding a natural listwise supervision structure. Rather than collapsing such data into a collection of independent winner-loser pairs, we score the full candidate set for each prompt offline using a pre-trained reward model, obtaining continuous-valued signals that reflect relative sample quality across the group.

Formally, for each prompt \(\rvc\), let \(\{\rvx_{0}^{(i)}\}_{i=1}^{N_c}\) denote a set of \(N_c\) candidate images associated with \(\rvc\), where \(N_c\) may vary across prompts, and let \(r_i = r(\rvc, \rvx_{0}^{(i)})\) denote the corresponding reward score for candidate \(\rvx_{0}^{(i)}\). Each training example is therefore a reward-labeled list
\[
\{(\rvx_{0}^{(i)}, r_i)\}_{i=1}^{N_c},
\]
from which we construct our reward-aware listwise preference optimization objective.

\subsection{Diffusion LAIR: Listwise Advantage-weighted Implicit Reward Optimization}

Our objective is motivated by the same high-level principle underlying prior diffusion preference optimization methods such as Diffusion-DPO and DSPO: the model should assign higher implicit reward (defined in Equation ~\ref{eq:implicit-reward}) to preferred samples and lower implicit reward to dispreferred ones. To generalize this preference-separation principle from pairwise comparisons to listwise supervision, we first transform the reward scores within each group into normalized weights. For a prompt \(\rvc\) with candidate set \(\{\rvx_0^{(i)}\}_{i=1}^{N_c}\) and reward scores \(\{r_i\}_{i=1}^{N_c}\), we define
\begin{equation}
    p_i
    :=
    \frac{\exp(r_i / \tau)}
         {\sum_{j=1}^{N_c} \exp(r_j / \tau)},
    \qquad
    w_i
    :=
    p_i - \frac{1}{N_c},
\end{equation}
where \(\tau > 0\) is a temperature parameter. Centering removes prompt-level offsets and turns the reward distribution into a relative advantage signal, so updates promote candidates above the group baseline and suppress candidates below it.

The softmax operation is chosen such that \(p_i\) measures the relative quality of sample \(i\) within the candidate set. This yields a normalized distribution that emphasizes high-reward samples and can then be centered against the uniform baseline \(1/N_c\) to obtain signed advantage weights. The centered weights satisfy \(\sum_{i=1}^{N_c} w_i = 0\), so samples with higher reward receive positive weight and samples with lower reward receive negative weight. 

Using these weights, we define our reward-aware listwise preference optimization objective as
\begin{equation}
    \mathcal{L}_{\mathrm{Diffusion-LAIR}}(\theta)
    =
    \mathbb{E}_{\rvc,\, t,\, \{\epsilon_i\}_{i=1}^{N_c}}
    \left[
        -\sum_{i=1}^{N_c} w_i \, s_{\theta}^{(i)}
        +
        \frac{\lambda}{N_c}
        \sum_{i=1}^{N_c} \left(s_{\theta}^{(i)}\right)^2
    \right],
    \label{eq:ours_loss}
\end{equation}
where \(s_{\theta}^{(i)} := s_{\theta}(\rvx_0^{(i)}, \rvx_t^{(i)}, t, \rvc, \epsilon_i)\), and \(\lambda > 0\) controls the strength of regularization. Conceptually, our method can be thought of as advantage-weighted regression on the implicit reward with a quadratic regularization term: the first term allocates learning signal across all candidates, encouraging \(s_{\theta}\) to increase for high-reward samples with \(w_i > 0\), and to decrease for low-reward samples with \(w_i < 0\); the second term penalizes large magnitudes of $s_\theta$, yielding conservative updates and preventing aggressive deviation from the reference.


\subsection{Theoretical Analysis} \label{sec:theory}

In this section, we study the theoretical properties of our objective. We first derive the closed-form optimal implicit reward implied by our loss. Under some assumptions, this optimal implicit reward implies a surrogate bound on the KL divergence between the reference distribution and induced implicit reward-tilted distribution; while not an exact guarantee, this provides intuition into how the regularization strength $\lambda$ controls the sharpness of the induced preference tilting.


\textbf{Optimal Implicit Reward}
We characterize the implicit reward that is optimal under our objective for a fixed prompt-level candidate set. 

\begin{prop}
For fixed listwise weights \(\{w_i\}_{i=1}^{N_c}\), the objective in Eq.~\eqref{eq:ours_loss} is strictly convex in \(\{s_i\}_{i=1}^{N_c}\) and admits the unique pointwise minimizer
\begin{equation}
    s_i^*
    =
    \frac{N_c}{2\lambda} w_i,
    \qquad i=1,\dots,N_c.
\end{equation}
\label{eq:optimal_s}
\end{prop}

The proof is deferred to Section \ref{sec:opt-proof} of the appendix. Proposition~\ref{eq:optimal_s} makes the behavior of our objective transparent. Samples with relatively higher reward receive positive weights \(w_i > 0\) and are assigned positive optimal implicit reward \(s_i^* > 0\), while samples with relatively lower reward receive negative optimal implicit reward. Thus, the model is encouraged to denoise high-reward samples better than the reference model and low-reward samples worse than the reference model.


Moreover, because the weights are centered, \(\sum_{i=1}^{N_c} w_i = 0\), we also have \(\sum_{i=1}^{N_c} s_i^* = 0\), showing that the objective redistributes implicit reward within each group. This is desirable because preference optimization is fundamentally a \emph{relative} problem: for a fixed prompt, the goal is not to make every candidate more preferred, but to shift probability mass toward higher-quality samples and away from lower-quality ones. By enforcing zero-sum redistribution within the group, the objective focuses learning on improving the ranking structure implied by the rewards.

\textbf{Surrogate KL Bound.}
The boundedness of the optimal implicit reward also provides a surrogate
distribution-shift interpretation, but this interpretation does not follow
from Proposition~\ref{eq:optimal_s} alone. Proposition~\ref{eq:optimal_s}
characterizes the pointwise optimum of the sampled implicit-reward
contributions on a finite prompt-level candidate set, whereas a KL divergence
is defined over the full conditional distribution of \(\rvx_0 \mid \rvc\).
Following the clean-level implicit reward definition from Equation \ref{eq:implicit-reward}, we define
the effective optimal implicit reward
\begin{equation}
    S^*(\rvx_0,\rvc)
    :=
    \mathbb{E}_{t,\epsilon}
    \left[
        s^*(\rvx_0,\rvx_t,t,\rvc,\epsilon)
    \right].
\end{equation}
We assume that this effective implicit reward admits a measurable full-support
extension over \(\operatorname{supp}(p_{\mathrm{ref}}(\cdot \mid \rvc))\)
satisfying
\begin{equation}
    a_{\rvc}
    \leq
    S^*(\rvx_0,\rvc)
    \leq
    b_{\rvc},
    \qquad
    b_{\rvc} - a_{\rvc}
    \leq
    \Delta_{\rvc}.
\end{equation}
Here, \(\Delta_{\rvc}\) denotes the range of the extended clean-level surrogate
implicit reward. This assumption should be interpreted as an idealized
bounded-extension condition: the finite-list pointwise optimum has range at
most \(N_c/(2\lambda)\), and when the full-support extension preserves this
finite-list range after marginalizing over noising variables, one may take
\(\Delta_{\rvc} = N_c/(2\lambda)\). More generally, the KL bound below holds
for any bounded extension with range \(\Delta_{\rvc}\).

Following the standard log-ratio interpretation used in diffusion preference
optimization \citep{wallace2024diffusion, xue2025advantage, bai2025towards},
we assume that the effective implicit reward approximates a scaled density
ratio between the learned model and the reference model:
\begin{equation}
    S^*(\rvx_0,\rvc)
    \approx
    \eta
    \log
    \frac{
        p^*(\rvx_0 \mid \rvc)
    }{
        p_{\mathrm{ref}}(\rvx_0 \mid \rvc)
    },
    \qquad
    \eta > 0.
\end{equation}
Motivated by this interpretation, we define the corresponding surrogate tilted
distribution
\begin{equation}
    \widetilde p^*(\rvx_0 \mid \rvc)
    =
    \frac{1}{Z(\rvc)}
    p_{\mathrm{ref}}(\rvx_0 \mid \rvc)
    \exp\left(
        \frac{S^*(\rvx_0,\rvc)}{\eta}
    \right),
\end{equation}
where \(Z(\rvc)\) is the normalizing constant. This yields the following
corollary.

\begin{cor}
Under the log-ratio approximation and the bounded full-support extension
assumption above, the surrogate tilted distribution induced by the effective
optimal implicit reward satisfies
\begin{equation}
    D_{\mathrm{KL}}\!\left(
        \widetilde p^*(\cdot \mid \rvc)
        \,\|\, 
        p_{\mathrm{ref}}(\cdot \mid \rvc)
    \right)
    \leq
    \frac{\Delta_{\rvc}}{\eta}.
\end{equation}
In particular, if the full-support extension preserves the range of the
finite-list optimum from Proposition~\ref{eq:optimal_s}, then
\(\Delta_{\rvc} = \frac{N_c}{2\lambda}\), giving
\begin{equation}
    D_{\mathrm{KL}}\!\left(
        \widetilde p^*(\cdot \mid \rvc)
        \,\|\, 
        p_{\mathrm{ref}}(\cdot \mid \rvc)
    \right)
    \leq
    \frac{N_c}{2\lambda\eta}.
\end{equation}
\end{cor}

The proof is deferred to Section~\ref{sec:kl-proof} of the appendix.
Intuitively, the corollary states that any bounded full-support extension of
the effective implicit reward with range \(\Delta_{\rvc}\) induces a surrogate
tilted distribution whose KL divergence from the reference model is at most
\(\Delta_{\rvc}/\eta\). Thus, when the extension preserves the finite-list
score range implied by Proposition~\ref{eq:optimal_s}, the regularization
parameter \(\lambda\) controls the sharpness of the surrogate preference tilt.
The bound scales linearly with \(N_c\) under our chosen normalization,
suggesting that \(\lambda\) may be adjusted with the candidate-set size to
control the effective magnitude of the update.

\textbf{Advantage Against Pairwise Objectives.}
Pairwise methods such as Diffusion DPO and DSPO reduce supervision to isolated winner-loser comparisons, discarding the ranking structure available when multiple candidates are observed for the same prompt. This has consequences for the implicit reward they induce: because the underlying pairwise logistic objective can continue decreasing as the preferred-dispreferred margin grows, it does not by itself impose a finite target on the induced implicit reward. A losing sample in a single noisy comparison can be pushed downward even if it is relatively strong within the prompt-level candidate set.

In contrast, our objective yields the finite closed-form optimum \(s_i^*\), assigning each sample a \(\lambda\)-controlled implicit reward according to its relative standing within the full candidate set. Thus, our method better preserves ranking structure and reward spacing: because continuous reward scores can provide a finer-grained signal than individual preference comparisons when the reward model is reliable, this listwise, reward-aware formulation better reflects relative sample quality while remaining robust to noisy or inconsistent comparisons. 


\section{Experiments} \label{sec:experiments}

\begin{figure}[!h]
    \centering
    \includegraphics[
        width=0.85\linewidth,
        height=0.85\textheight,
        keepaspectratio
    ]{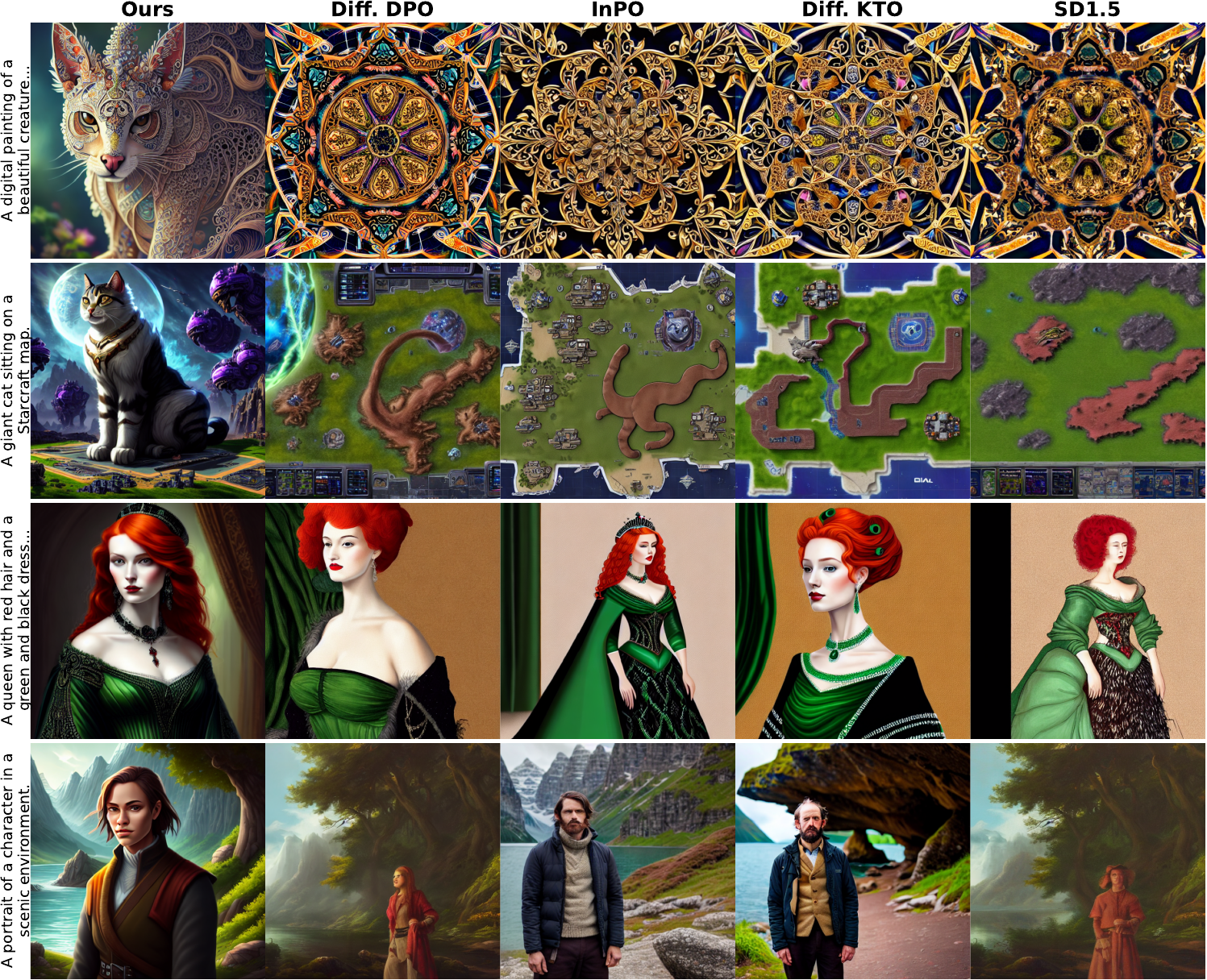}
    \caption{Comparison of images generated by LAIR-SD1.5 and baseline methods. Qualitatively, LAIR-SD1.5 often shows better prompt adherence while maintaining strong visual quality.}
    \label{fig:ours-geneval}
\end{figure}

\subsection{Experimental Setup}

\textbf{Dataset and Models.} We fine-tune Stable Diffusion 1.5 (SD1.5) and Stable Diffusion XL (SDXL) with the Diffusion-LAIR objective on the Pick-a-Pic v2 \citep{kirstain2023pick} dataset. While the standard Pick-a-Pic v2 (Pick V2) dataset consists of image pairs per prompt, we observe that each unique prompt may have multiple corresponding pairs, which we aggregate to form lists of images per prompt. Lists per prompt range from size 2 to $N$, where $N$ is the maximum list size hyperparameter chosen to balance memory overhead and effective listwise supervision. Prior to training, we score all images in the Pick V2 dataset with the PickScore \citep{kirstain2023pick} reward model. To evaluate model alignment with human preferences, we use test prompts from the HPD \citep{wu2023human} and Parti-prompts \citep{yu2022scaling} datasets. Additionally, we evaluate image editing using the InstructPix2Pix dataset \citep{brooks2023instructpix2pix} and compositional generation with the GenEval benchmarking suite \citep{ghosh2023geneval}.

\textbf{Baselines.} We compare our fine-tuned models' generated images against those from existing methods. To measure both human preference alignment and image editing performance, we include comparisons against recent state-of-the-art methods including Diffusion DPO \citep{wallace2024diffusion}, DSPO \citep{zhu2025dspo}, Diffusion KTO \citep{li2024aligning}, MaPO \citep{hong2026margin}, and InPO \citep{lu2025inpo}. We use the official checkpoints for Diffusion DPO, Diffusion KTO, MaPO, and InPO. We fine-tune SD1.5 with the DSPO objective using the official codebase. Due to the high computational cost of applying DSPO at SDXL scale, we omit it from our SDXL comparisons. Additionally, we directly report baseline results on GenEval from \cite{sun2026craft}, including comparison with SmPO \citep{lu2025smoothed}, SPO \citep{liang2025aesthetic}, and CRAFT \citep{sun2026craft}. Our evaluation and the reported GenEval baseline evaluation both use the official GenEval codebase's generation and scoring scripts, so we are confident that results can be fairly compared.

\textbf{Evaluation.}
Following the experimental setup of \cite{zhu2025dspo}, we evaluate human preference alignment by automatically scoring text-to-image generations with reward models. Specifically, we employ PickScore \citep{kirstain2023pick}, HPS v2 \citep{wu2023human}, CLIP \citep{radford2021learning}, LAION Aesthetics Score \citep{schuhmann2022laionaesthetics}, and ImageReward \citep{xu2023imagereward}. For the text-to-image generation and instruction-based image editing experiments, we use the default hyperparameters used by \cite{wallace2024diffusion, zhu2025dspo}. We sample five independent images per prompt and report average scores, keeping the random seeds the same for each baseline. For GenEval evaluation, we use the official codebase and its default hyperparameters. 

\textbf{Computation Cost.}
We train SD1.5 and SDXL using 2 and 3 A100 GPUs, respectively. The total training time is about $\sim$24 GPU hours for SD1.5 and $\sim$140 A100 GPU hours for SDXL. Our method is substantially more efficient than standard baselines like Diffusion DPO and DSPO, which require up to nearly $5$ times more H100 GPU hours on SDXL training, as reported by \cite{sun2026craft}. For more specific training, hyperparameter, and evaluation details, please refer to Section \ref{sec:exp-details}. Additional ablations on reward noise robustness and hyperparameter sweeps are in Section \ref{sec:ablations}.

\subsection{Main Experimental Results}

\begin{table*}[t]
\centering
\caption{Reward score evaluation for human preference alignment on SD1.5 and SDXL models. Average scores for five independent samples per prompt are reported. Best results are \textbf{bolded}, and second-best results are \underline{underlined}. LAIR improves not only the reward model used for training but also transfers to independent reward metrics, compositional generation, and image editing.}
\label{tab:t2i_scores_sd15_sdxl_side_by_side}
\renewcommand{\arraystretch}{1.06}
\setlength{\tabcolsep}{3pt}

\begin{subtable}[t]{0.49\textwidth}
\centering
\caption{SD1.5 models.}
\label{tab:sd15_t2i_scores_compact}
\scriptsize
\resizebox{\linewidth}{!}{
\begin{tabular}{c|c|ccccc}
\toprule
\textbf{Dataset} & \textbf{Method} & \textbf{Pick} & \textbf{HPS} & \textbf{Aes.} & \textbf{CLIP} & \textbf{IR} \\
\midrule
\multirow{6}{*}{Parti-prompt}
& SD1.5      & 21.243 & 0.2738 & 5.360 & 0.3322 & 0.1653 \\
& DSPO       & 21.521 & 0.2813 & \underline{5.658} & 0.3376 & 0.5763 \\
& InPO       & \underline{21.735} & \underline{0.2842} & 5.613 & \underline{0.3457} & \underline{0.7203} \\
& Diff.-DPO  & 21.497 & 0.2773 & 5.445 & 0.3373 & 0.3539 \\
& Diff.-KTO  & 21.550 & 0.2825 & 5.568 & 0.3390 & 0.5941 \\
\rowcolor{blue!10}
& Diff.-LAIR & \textbf{21.992} & \textbf{0.2860} & \textbf{5.671} & \textbf{0.3485} & \textbf{0.8107} \\
\midrule
\multirow{6}{*}{HPD}
& SD1.5      & 20.792 & 0.2723 & 5.567 & 0.3498 & 0.0988 \\
& DSPO       & 21.426 & 0.2833 & \underline{5.873} & 0.3566 & 0.6612 \\
& InPO       & \underline{21.792} & \underline{0.2862} & \underline{5.873} & \underline{0.3659} & \underline{0.8114} \\
& Diff.-DPO  & 21.226 & 0.2767 & 5.688 & 0.3553 & 0.3087 \\
& Diff.-KTO  & 21.430 & 0.2849 & 5.792 & 0.3572 & 0.6835 \\
\rowcolor{blue!10}
& Diff.-LAIR & \textbf{22.068} & \textbf{0.2871} & \textbf{5.904} & \textbf{0.3670} & \textbf{0.8382} \\
\bottomrule
\end{tabular}
}
\end{subtable}
\hfill
\begin{subtable}[t]{0.49\textwidth}
\centering
\caption{SDXL models.}
\label{tab:sdxl_t2i_scores_compact}
\scriptsize
\resizebox{\linewidth}{!}{
\begin{tabular}{c|c|ccccc}
\toprule
\textbf{Dataset} & \textbf{Method} & \textbf{Pick} & \textbf{HPS} & \textbf{Aes.} & \textbf{CLIP} & \textbf{IR} \\
\midrule
\multirow{5}{*}{Parti-prompt}
& SDXL       & 22.425 & 0.2843 & 5.809 & 0.356 & 0.776 \\
& InPO       & \underline{22.723} & \underline{0.2908} & \underline{5.872} & 0.359 & 1.045 \\
& Diff.-DPO  & 22.693 & 0.2894 & 5.824 & \underline{0.365} & \underline{1.066} \\
& MaPO       & 22.399 & 0.2861 & \textbf{5.957} & 0.354 & 0.873 \\
\rowcolor{blue!10}
& Diff.-LAIR & \textbf{22.765} & \textbf{0.2920} & 5.845 & \textbf{0.366} & \textbf{1.104} \\
\midrule
\multirow{5}{*}{HPD}
& SDXL       & 22.589 & 0.2860 & 6.134 & 0.382 & 0.865 \\
& InPO       & \underline{23.001} & \underline{0.2939} & \underline{6.186} & 0.384 & 1.074 \\
& Diff.-DPO  & 22.949 & 0.2921 & 6.141 & \underline{0.388} & \underline{1.082} \\
& MaPO       & 22.638 & 0.2902 & \textbf{6.245} & 0.382 & 0.968 \\
\rowcolor{blue!10}
& Diff.-LAIR & \textbf{23.033} & \textbf{0.2950} & 6.152 & \textbf{0.391} & \textbf{1.130} \\
\bottomrule
\end{tabular}
}
\end{subtable}

\end{table*}

\textbf{General Preference Alignment Results.}
We present comprehensive evaluation of T2I generations on the Parti-Prompt and HPD datasets in Table \ref{tab:t2i_scores_sd15_sdxl_side_by_side}. Our method remains competitive with all baselines, outperforming them on most metrics. Notably, despite only seeing the PickScore reward during training, our models substantially outperform SOTA baselines like DSPO and InPO in other metrics such as ImageReward, demonstrating the ability of our objective to improve general sample quality in a way that transfers beyond the specific reward model used during training. This trend is consistent across both SD1.5 and SDXL, suggesting that the benefits of our listwise reward-weighted objective scale to stronger base models. Additional qualitative results can be found in Section \ref{sec:additional-results}.


\begin{figure}[!htbp]
    \centering
    \begin{minipage}[c]{0.63\linewidth}
        \centering
        \includegraphics[
            width=\linewidth,
            height=0.38\textheight,
            keepaspectratio,
            trim={0.2in 0.2in 0.2in 0.2in},
            clip
        ]{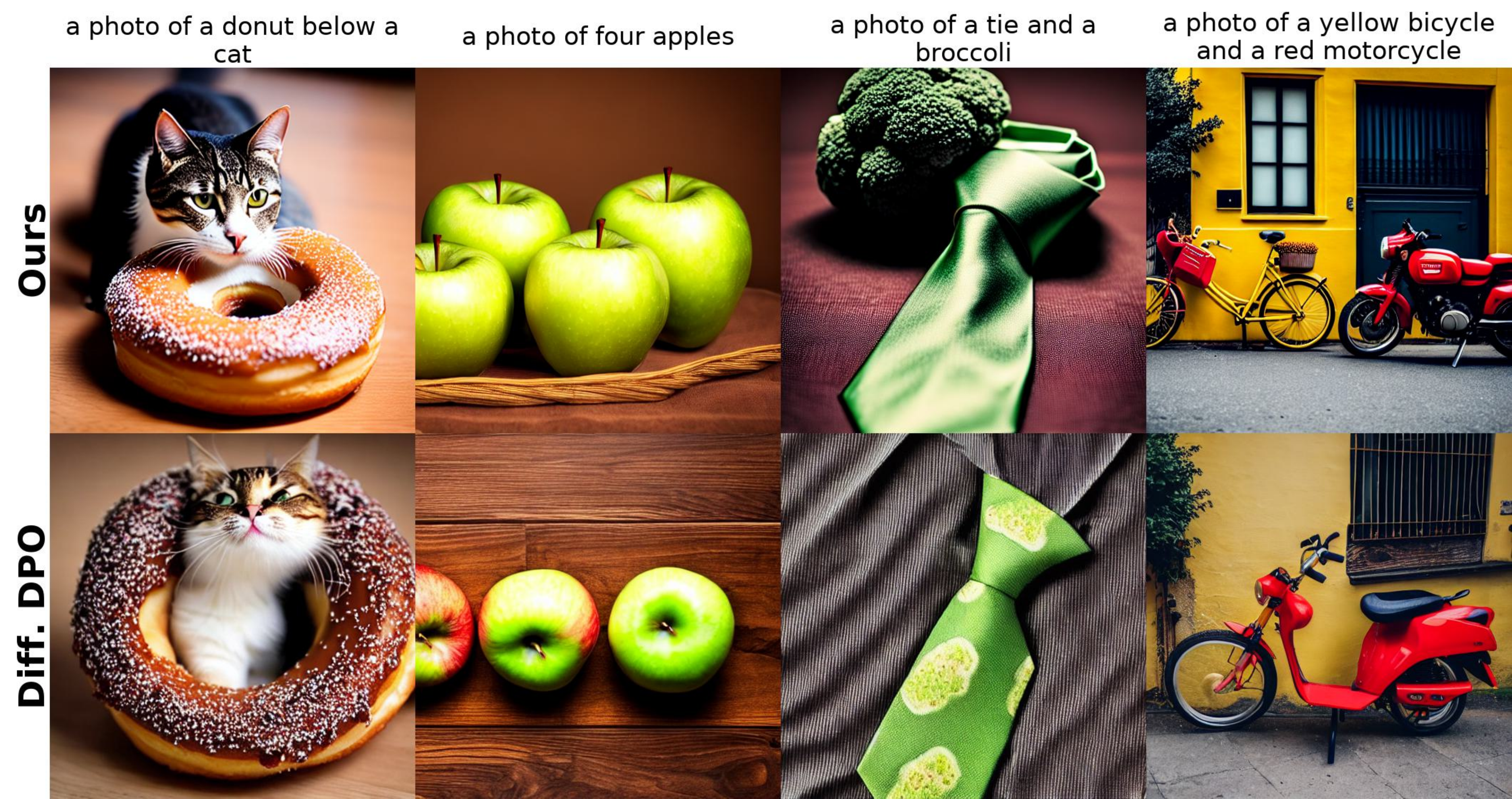}
    \end{minipage}
    \hfill
    \begin{minipage}[c]{0.30\linewidth}
        \caption{Qualitative comparison on GenEval prompts. LAIR-SD1.5 produces generations that better satisfy compositional constraints than Diffusion DPO across representative object, attribute, and counting examples.}
        \label{fig:geneval-compare-baselines}
    \end{minipage}
\end{figure}

\begin{table*}[t]
\centering
\caption{GenEval evaluation results for SD1.5- and SDXL-based models. Best results are \textbf{bolded}, and second-best results are \underline{underlined}.}
\label{tab:geneval_side_by_side}
\renewcommand{\arraystretch}{1.06}
\setlength{\tabcolsep}{3pt}

\begin{subtable}[t]{0.49\textwidth}
\centering
\caption{SD1.5-based models.}
\label{tab:geneval_sd15_compact}
\scriptsize
\resizebox{\linewidth}{!}{
\begin{tabular}{c|ccccccc}
\toprule
\textbf{Method} & \textbf{Overall} & \textbf{Color} & \textbf{Count} & \textbf{Pos.} & \textbf{Single} & \textbf{Attr.} & \textbf{Two} \\
\midrule
SD1.5         & 42.42 & 74.47 & 36.56 & 3.50 & 95.62 & 5.25 & 39.14 \\
SmPO          & 42.86 & 75.53 & 32.19 & 4.00 & 96.25 & 7.25 & 41.92 \\
Diff.-DPO     & 43.28 & 77.39 & 35.00 & 4.75 & \underline{98.12} & 6.00 & 38.38 \\
CRAFT         & \underline{45.82} & \underline{77.66} & \underline{38.12} & \textbf{6.75} & 96.88 & \underline{8.25} & 47.22 \\
SPO           & 44.04 & 72.61 & 33.44 & \underline{6.25} & 95.94 & 6.25 & \underline{49.75} \\
\rowcolor{blue!10}
Diff.-LAIR    & \textbf{51.44} & \textbf{84.31} & \textbf{45.62} & \underline{6.25} & \textbf{99.38} & \textbf{13.25} & \textbf{59.85} \\
\bottomrule
\end{tabular}
}
\end{subtable}
\hfill
\begin{subtable}[t]{0.49\textwidth}
\centering
\caption{SDXL-based models.}
\label{tab:geneval_sdxl_compact}
\scriptsize
\resizebox{\linewidth}{!}{
\begin{tabular}{c|ccccccc}
\toprule
\textbf{Method} & \textbf{Overall} & \textbf{Color} & \textbf{Count} & \textbf{Pos.} & \textbf{Single} & \textbf{Attr.} & \textbf{Two} \\
\midrule
SDXL        & 55.05 & \underline{88.30} & 42.81 & 11.00 & 97.50 & 21.00 & 70.96 \\
Diff.-DPO   & 57.23 & 86.70 & \textbf{45.62} & 11.00 & 98.75 & 20.75 & 80.56 \\
SmPO        & 57.86 & 86.70 & \underline{44.69} & 10.50 & 98.75 & \underline{27.00} & 79.55 \\
SPO         & 55.43 & 84.04 & 38.44 & 11.75 & 97.81 & 20.50 & 80.05 \\
CRAFT       & \underline{57.97} & 87.23 & 36.88 & \textbf{14.50} & \underline{99.06} & 23.75 & \textbf{86.36} \\
\rowcolor{blue!10}
Diff.-LAIR  & \textbf{59.16} & \textbf{90.96} & 39.69 & \underline{13.50} & \textbf{100.00} & \textbf{28.25} & \underline{82.58} \\
\bottomrule
\end{tabular}
}
\end{subtable}

\vspace{0.25em}
{\footnotesize Baseline results are taken directly from \cite{sun2026craft}.}

\end{table*}

\textbf{GenEval and Image Editing Results.}
We next test whether LAIR generalizes beyond the PickScore supervision used during training. Despite being trained only on PickScore-labeled preference data, LAIR achieves stronger compositional performance than prior baselines on GenEval across both SD1.5 and SDXL scales (Table~\ref{tab:geneval_side_by_side}). It also transfers effectively to image editing: following \cite{zhu2025dspo,bai2025towards}, we evaluate SDEdit~\citep{meng2021sdedit} edits on 1000 InstructPix2Pix examples, where LAIR-SDXL outperforms baselines on most metrics (Table~\ref{tab:instructpix2pix_sdxl_winrates}). These results demonstrate strong cross-reward generalization: optimizing only PickScore yields gains in independent compositional and editing evaluations, suggesting that LAIR improves broader semantic alignment rather than merely overfitting the training reward.

\subsection{How Do Cardinal and Ordinal Information Contribute?}
\label{sec:cardinal_ordinal}

We study whether LAIR's gains primarily come from identifying preferred samples or from preserving cardinal reward gaps. For each prompt $c$ with reward-scored candidates $\{(\rvx_0^{(i)}, r_i)\}_{i=1}^{N_c}$, we compare standard LAIR against variants that remove reward magnitude information while keeping the remaining training setup fixed.

\textbf{Ablation setup}
In \emph{rank-only LAIR}, we replace each reward by an equally spaced value determined by its rank,
$
\tilde{r}_i = \frac{N_c-\operatorname{rank}(i)}{N_c-1},
$
preserving the full ordering but discarding cardinal gaps. In \emph{top-1 LAIR}, we retain only the winner identity,
$
\tilde{r}_i = \mathbf{1}\{i=\arg\max_j r_j\}.
$
We also vary the LAIR temperature $\tau$; as $\tau\rightarrow 0$, the softmax weights become increasingly concentrated on the highest-reward candidate. Finally, we compare against RankDPO~\citep{karthik2025scalable}, which optimizes an ordinal listwise ranking objective. Since official RankDPO code is unavailable, we implement it in our training framework.

\begin{table*}[t]
    \centering
    \caption{
        Additional SD1.5 ablations and image editing evaluation.
        Left: effect of cardinal and ordinal reward information on Parti-Prompts.
        Right: InstructPix2Pix win rates against SDXL.
    }
    \label{tab:cardinal_ordinal_editing}
    \renewcommand{\arraystretch}{1.06}
    \setlength{\tabcolsep}{3pt}

    \begin{subtable}[t]{0.57\textwidth}
        \centering
        \caption{Cardinal vs.\ ordinal reward information.}
        \label{tab:cardinal_ordinal}
        \scriptsize
        \resizebox{\linewidth}{!}{
        \begin{tabular}{lccccc}
            \toprule
            \textbf{Method} & \textbf{PickScore} & \textbf{HPS} & \textbf{Aes.} & \textbf{CLIP} & \textbf{IR} \\
            \midrule
            RankDPO
                & 21.642 & 0.2781 & 5.4710 & 0.3381 & 0.3979 \\
            Rank-only LAIR
                & 21.726 & 0.2822 & 5.5270 & 0.3431 & 0.5794 \\
            LAIR ($\tau=1.0$)
                & 21.978 & 0.2858 & 5.6500 & 0.3482 & 0.8011 \\
            \rowcolor{blue!10}
            LAIR ($\tau=0.05$)
                & \textbf{21.992} & \textbf{0.2860} & \textbf{5.6710}
                & \textbf{0.3485} & \textbf{0.8107} \\
            LAIR ($\tau=0.01$)
                & 21.936 & 0.2852 & 5.6287 & 0.3471 & 0.7748 \\
            Top-1 LAIR
                & 21.935 & 0.2852 & 5.6288 & 0.3475 & 0.7802 \\
            \bottomrule
        \end{tabular}
        }

    \end{subtable}
    \hfill
    \begin{subtable}[t]{0.40\textwidth}
        \centering
        \caption{InstructPix2Pix win rates against SDXL.}
        \label{tab:instructpix2pix_sdxl_winrates}
        \scriptsize
        \resizebox{\linewidth}{!}{
        \begin{tabular}{lccccc}
            \toprule
            \textbf{Method} & \textbf{Pick} & \textbf{HPS} & \textbf{Aes.} & \textbf{CLIP} & \textbf{IR} \\
            \midrule
            Diff.-DPO
                & \underline{81.2\%} & 76.5\% & 65.7\% & \underline{71.6\%} & 73.3\% \\
            InPO
                & 78.4\% & \underline{82.6\%} & 73.4\% & 62.4\% & \underline{74.8\%} \\
            MaPO
                & 48.8\% & 40.6\% & \textbf{88.4\%} & 34.8\% & 46.9\% \\
            \rowcolor{blue!10}
            Diff.-LAIR
                & \textbf{86.4\%} & \textbf{86.1\%} & \underline{81.6\%} & \textbf{77.5\%} & \textbf{81.0\%} \\
            \bottomrule
        \end{tabular}
        }
    \end{subtable}
\end{table*}

\textbf{Results}
Table~\ref{tab:cardinal_ordinal} shows that Top-1 LAIR and LAIR with $\tau=0.01$ perform nearly identically, confirming that low temperatures induce a winner-dominated training signal. Their strong performance indicates that winner identification accounts for a substantial portion of LAIR's gains. However, the full-score variants with $\tau=0.05$ and $\tau=1.0$ consistently outperform Top-1 across all five metrics, suggesting an additional benefit from cardinal reward gaps.

Rank-only LAIR performs worse than full LAIR, indicating that preserving the ordering alone does not recover the benefit of continuous rewards. At the same time, rank-only LAIR consistently outperforms RankDPO, suggesting that LAIR's candidate-wise target-regression objective contributes beyond ordinal listwise supervision. Overall, these results indicate that LAIR benefits from both strong winner identification and the graded information contained in continuous reward scores.

\section{Conclusion}

We propose Diffusion LAIR, a novel reward-aware listwise objective for offline preference optimization of diffusion models. Our method directly leverages groups of candidate generations and their continuous reward scores, enabling the model to learn from the full structure of relative sample quality within each prompt. Strong empirical results across several domains highlight the value of listwise reward-aware optimization as a simple and effective direction for aligning diffusion models with human preferences.



\newpage

\bibliography{ref}
\bibliographystyle{reference_style}


\appendix

\newpage

\section{Proofs}

\subsection{Optimal Implicit Reward Proof} \label{sec:opt-proof}

\begin{proof}[Proof of Proposition~\ref{eq:optimal_s}]
For a fixed prompt \(\rvc\), timestep \(t\), noise realizations \(\{\epsilon_i\}_{i=1}^{N_c}\), and listwise weights \(\{w_i\}_{i=1}^{N_c}\), the objective in Eq.~\eqref{eq:ours_loss} reduces to the deterministic function
\begin{equation}
    \mathcal{J}(\{s_i\}_{i=1}^{N_c})
    =
    - \sum_{i=1}^{N_c} w_i s_i
    +
    \frac{\lambda}{N_c} \sum_{i=1}^{N_c} s_i^2 .
\end{equation}
Since \(\mathcal{J}\) is separable across the variables \(\{s_i\}_{i=1}^{N_c}\), each \(s_i\) can be optimized independently. Taking the derivative with respect to \(s_i\) gives
\begin{equation}
    \frac{\partial \mathcal{J}}{\partial s_i}
    =
    -w_i + \frac{2\lambda}{N_c} s_i .
\end{equation}
Setting the derivative equal to zero yields
\begin{equation}
    -w_i + \frac{2\lambda}{N_c} s_i = 0
    \qquad \Longrightarrow \qquad
    s_i^* = \frac{N_c}{2\lambda} w_i,
\end{equation}
for each \(i=1,\dots,N_c\).

It remains to show uniqueness. The Hessian of \(\mathcal{J}\) with respect to \((s_1,\dots,s_{N_c})\) is
\begin{equation}
    \nabla^2 \mathcal{J}
    =
    \frac{2\lambda}{N_c} I_{N_c},
\end{equation}
which is positive definite for \(\lambda > 0\). Therefore, \(\mathcal{J}\) is strictly convex in \(\{s_i\}_{i=1}^{N_c}\), and the critical point above is the unique global minimizer.
\end{proof}

\subsection{Surrogate KL Bound Proof}
\label{sec:kl-proof}

\begin{proof}[Proof of Corollary~1]
Fix a prompt \(\rvc\), and let
\begin{equation}
    f(\rvx_0,\rvc)
    :=
    \frac{S^*(\rvx_0,\rvc)}{\eta}.
\end{equation}
By assumption, \(S^*(\rvx_0,\rvc)\) admits a measurable full-support extension satisfying
\begin{equation}
    a_{\rvc}
    \leq
    S^*(\rvx_0,\rvc)
    \leq
    b_{\rvc},
\end{equation}
for all \(\rvx_0 \in \operatorname{supp}(p_{\mathrm{ref}}(\cdot \mid \rvc))\). Hence
\begin{equation}
    \frac{a_{\rvc}}{\eta}
    \leq
    f(\rvx_0,\rvc)
    \leq
    \frac{b_{\rvc}}{\eta}.
\end{equation}

By definition, the surrogate tilted distribution is
\begin{equation}
    \widetilde p^*(\rvx_0 \mid \rvc)
    =
    \frac{1}{Z(\rvc)}
    p_{\mathrm{ref}}(\rvx_0 \mid \rvc)
    \exp\!\bigl(f(\rvx_0,\rvc)\bigr),
\end{equation}
where
\begin{equation}
    Z(\rvc)
    =
    \mathbb{E}_{p_{\mathrm{ref}}(\cdot \mid \rvc)}
    \left[
        \exp\!\bigl(f(\rvx_0,\rvc)\bigr)
    \right].
\end{equation}
Since \(f(\rvx_0,\rvc)\in [a_{\rvc}/\eta,\, b_{\rvc}/\eta]\), we have
\begin{equation}
    \exp\!\left(\frac{a_{\rvc}}{\eta}\right)
    \leq
    Z(\rvc)
    \leq
    \exp\!\left(\frac{b_{\rvc}}{\eta}\right),
\end{equation}
and therefore
\begin{equation}
    \frac{a_{\rvc}}{\eta}
    \leq
    \log Z(\rvc)
    \leq
    \frac{b_{\rvc}}{\eta}.
    \label{eq:logZ_bounds}
\end{equation}

Next, by the definition of KL divergence,
\begin{align}
    D_{\mathrm{KL}}\!\left(
        \widetilde p^*(\cdot \mid \rvc)
        \,\|\, 
        p_{\mathrm{ref}}(\cdot \mid \rvc)
    \right)
    &=
    \mathbb{E}_{\widetilde p^*(\cdot \mid \rvc)}
    \left[
        \log \frac{\widetilde p^*(\rvx_0 \mid \rvc)}
        {p_{\mathrm{ref}}(\rvx_0 \mid \rvc)}
    \right] \\
    &=
    \mathbb{E}_{\widetilde p^*(\cdot \mid \rvc)}
    \left[
        f(\rvx_0,\rvc) - \log Z(\rvc)
    \right] \\
    &=
    \mathbb{E}_{\widetilde p^*(\cdot \mid \rvc)}
    \left[
        f(\rvx_0,\rvc)
    \right]
    -
    \log Z(\rvc).
\end{align}
Using the upper bound \(f(\rvx_0,\rvc) \leq b_{\rvc}/\eta\) and the lower bound on \(\log Z(\rvc)\) from Eq.~\eqref{eq:logZ_bounds}, we obtain
\begin{equation}
    D_{\mathrm{KL}}\!\left(
        \widetilde p^*(\cdot \mid \rvc)
        \,\|\, 
        p_{\mathrm{ref}}(\cdot \mid \rvc)
    \right)
    \leq
    \frac{b_{\rvc}}{\eta}
    -
    \frac{a_{\rvc}}{\eta}
    =
    \frac{b_{\rvc}-a_{\rvc}}{\eta}
    \leq
    \frac{\Delta_{\rvc}}{\eta}.
\end{equation}
This proves the first claim.

For the second claim, if the full-support extension preserves the range of the finite-list optimum from Proposition~\ref{eq:optimal_s}, then
\begin{equation}
    \Delta_{\rvc}
    =
    \frac{N_c}{2\lambda},
\end{equation}
and therefore
\begin{equation}
    D_{\mathrm{KL}}\!\left(
        \widetilde p^*(\cdot \mid \rvc)
        \,\|\, 
        p_{\mathrm{ref}}(\cdot \mid \rvc)
    \right)
    \leq
    \frac{N_c}{2\lambda\eta}.
\end{equation}
This completes the proof.
\end{proof}

\paragraph{Boundedness on the observed candidate set.}
We now justify the finite-list range used in the second claim. For the observed candidate set
\(\{\rvx_0^{(i)}\}_{i=1}^{N_c}\), Proposition~\ref{eq:optimal_s} gives the pointwise optimum of the sampled implicit-reward contribution as
\begin{equation}
    s_i^*(t,\epsilon_i)
    =
    \frac{N_c}{2\lambda}w_i,
    \qquad
    i=1,\dots,N_c.
\end{equation}
This target is independent of the noising variables \(t\) and \(\epsilon_i\). Therefore, the corresponding clean-level optimal implicit reward on the observed candidate set is
\begin{equation}
    S_i^*
    :=
    S^*(\rvx_0^{(i)},\rvc)
    =
    \mathbb{E}_{t,\epsilon_i}
    \left[
        s_i^*(t,\epsilon_i)
    \right]
    =
    \frac{N_c}{2\lambda}w_i.
\end{equation}
Since \(w_i = p_i - 1/N_c\) and \(0 \leq p_i \leq 1\), each observed clean-level target satisfies
\begin{equation}
    -\frac{1}{2\lambda}
    \leq
    S_i^*
    \leq
    \frac{N_c-1}{2\lambda}.
\end{equation}
Moreover, the range of the clean-level optimal implicit reward over the observed candidate set is
\begin{align}
    \max_i S_i^* - \min_i S_i^*
    &=
    \frac{N_c}{2\lambda}
    \left(
        \max_i w_i - \min_i w_i
    \right) \\
    &=
    \frac{N_c}{2\lambda}
    \left(
        \max_i p_i - \min_i p_i
    \right)
    \leq
    \frac{N_c}{2\lambda}.
\end{align}
Thus, marginalizing the pointwise optimum over the noising variables preserves the same finite-list target and yields a clean-level range of at most \(N_c/(2\lambda)\) on the observed candidate set. This finite-list boundedness alone does not imply a full-distribution KL bound; the bounded full-support extension assumption above is what extends this range control from the observed candidate set to \(\operatorname{supp}(p_{\mathrm{ref}}(\cdot\mid\rvc))\). While idealized, this assumption provides a useful surrogate interpretation of how the finite-list regularization induced by \(\lambda\) can translate into controlled preference tilting at the distribution level.

\section{Experiment Details} \label{sec:exp-details}

\subsection{Dataset and Training Details}

\paragraph{Pick-a-Pic v2 Dataset} Following works like Diffusion DPO and DSPO, we employ the Pick-a-Pic v2 \citep{kirstain2023pick} dataset during training. Pick-a-Pic is a large-scale, publicly available dataset consisting of preferred-dispreferred image pairs. Each entry corresponds to a text prompt, two images, and a binary preference label for each image.

We observe that not every entry in Pick-a-Pic corresponds to a unique prompt. In fact, even though the dataset contains roughly 1 million entries, there are only about $\sim59$k unique prompts total. This allows us to reshape the dataset into prompt-level candidate sets, where each training example consists of a prompt $\rvc$ together with a list of images generated for that prompt. Specifically, for each unique prompt, we aggregate all associated images and their corresponding reward information (scored offline), forming groups of the form ${(\rvx_i, r_i)}_{i=1}^{N_c}$, where $N_c$ denotes the number of available candidates for prompt $\rvc$. In practice, we can subsample or truncate these groups so that each list contains between $2$ and a maximum size $N$ images, enabling listwise preference optimization over multiple candidates for the same prompt rather than treating each pairwise comparison as an independent training example.

\paragraph{Training Details}

We use the AdamW optimizer for training SD1.5 and SDXL. SD1.5 training runs on 2 A100 GPUs, while SDXL training runs on 3 A100 GPUs. We finetune SD1.5 for a total of 2300 gradient update steps, and SDXL for 1500 steps. The specific hyperparameter setup is provided in the following section.

\subsection{Training Hyperparameters}

We provide the detailed hyperparameter configurations for SD1.5 and SDXL training in Table \ref{tab:hyperparameters}. Note that the maximum list size $N$ is reduced for SDXL training to reduce memory overhead.

\begin{table}[h]
\centering
\caption{Hyperparameter settings for SD1.5 and SDXL experiments.}
\label{tab:hyperparameters}
\renewcommand{\arraystretch}{1.12}
\setlength{\tabcolsep}{12pt}
\begin{tabular}{c|cc}
\toprule
\textbf{Hyperparameter} & \textbf{SD1.5} & \textbf{SDXL} \\
\midrule
Learning rate & $3 \times10^{-6}$ & $6 \times 10^{-7}$ \\
Maximum list size $N$ & 30 & 10 \\
Reward temperature $\tau$ & 0.05 & 0.5 \\
Per-GPU Batch size & 1 & 1 \\
Gradient accumulation steps & 16 & 16 \\
Regularization strength $\lambda$ & 0.00025 & 0.00025 \\
CFG prompt dropout proportion & 0.1 & 0.1 \\
\bottomrule
\end{tabular}
\end{table}

\subsection{Evaluation Details}

\paragraph{General Human Preference} We evaluate our model and baselines for general human preference by generating images conditioned on prompts from the Parti-prompts \citep{yu2022scaling} and HPD \citep{wu2023human} datasets. Parti-prompts consists of 1632 prompts and HPD consists of 3200 prompts. For each prompt, we generate 5 images and report average reward scores over the entire dataset. Random seeds are standardized across evaluation runs to ensure fair comparison between baselines. 

\paragraph{GenEval} We evaluate our model on the GenEval benchmark directly from the official codebase. We generate 4 images per prompt and report average scores using the official evaluation script. The baseline results are reported directly from \cite{sun2026craft}, which states that official code and checkpoints are used.

\paragraph{Image Editing} We evaluate image editing capabilities using the InstructPix2Pix \citep{brooks2023instructpix2pix} dataset. We randomly sample 1000 image-prompt pairs from InstructPix2Pix and use SDEdit \citep{meng2021sdedit} with a noise strength of 0.6 to generate image edits. For each prompt we generate 5 images, with random seeds standardized across evaluation runs, and compute average reward scores. The win rate of a model against SDXL is computed as the ratio of image-prompt pairs for which the model's average reward is greater than SDXL's average reward.

\section{Ablation Studies} \label{sec:ablations}

\subsection{Hyperparameter Sweeps}

We conduct three ablation studies to explore the effects of the maximum list size, reward temperature, and regularization scale on model performance. Fixing all other hyperparameters, we vary $N$, $\tau$, and $\lambda$ across SD1.5 training runs and evaluate on a random subset of 100 prompts from Parti-Prompt. 

\textbf{Effect of $N$} As seen in Table \ref{tab:ablation_parti_prompt_sd15}, larger groups generally perform at least as well as smaller groups, with some modest gains in rewards like PickScore and Aesthetics. This suggests that listwise supervision is beneficial, while the method is not overly sensitive to the precise group-size cap. 

\textbf{Effect of $\tau$} We find variation in the reward temperature does not seem to elicit large performance gaps, yet hypothesize that it may still be a useful tuning knob depending on the type of reward(s) used during training. 

\textbf{Effect of $\lambda$} A larger value of $\lambda = 0.00125$ still achieves strong performance relative to baselines, but underperforms the current LAIR result. However, a significantly smaller value of $\lambda = 0.00005$ results in worse performance compared to $\lambda = 0.00025, 0.00125$. This suggests that $\lambda$ is an important hyperparameter that must be carefully chosen to balance preference improvement and regularization.

\begin{table}[h]
\centering
\caption{Ablation results on 100 randomly sampled Parti-Prompt prompts using SD1.5. We report average reward scores across different rewards for 5 independent samples per prompt.}
\label{tab:ablation_parti_prompt_sd15}
\renewcommand{\arraystretch}{1.12}
\setlength{\tabcolsep}{4pt}

\begin{subtable}[t]{0.49\textwidth}
\centering
\caption{Effect of group size $N$.}
\label{tab:ablation_group_size}
\resizebox{\linewidth}{!}{
\begin{tabular}{c|ccccc}
\toprule
\textbf{Method} & \textbf{PickScore} & \textbf{HPS} & \textbf{Aesthetics} & \textbf{CLIP} & \textbf{ImageReward} \\
\midrule
$N=30$ & 21.950 & 0.2847 & 5.645 & 0.351 & 0.748 \\
$N=16$ & 21.960 & 0.2849 & 5.601 & 0.348 & 0.785 \\
$N=8$  & 21.909 & 0.2845 & 5.613 & 0.346 & 0.740 \\
$N=2$  & 21.886 & 0.2840 & 5.580 & 0.348 & 0.718 \\
\bottomrule
\end{tabular}
}
\end{subtable}
\hfill
\begin{subtable}[t]{0.49\textwidth}
\centering
\caption{Effect of reward temperature $\tau$.}
\label{tab:ablation_tau}
\resizebox{\linewidth}{!}{
\begin{tabular}{c|ccccc}
\toprule
\textbf{Method} & \textbf{PickScore} & \textbf{HPS} & \textbf{Aesthetics} & \textbf{CLIP} & \textbf{ImageReward} \\
\midrule
$\tau=1.0$  & 21.930 & 0.2850 & 5.594 & 0.351 & 0.765 \\
$\tau=0.5$  & 21.937 & 0.2845 & 5.631 & 0.350 & 0.779 \\
$\tau=0.05$ & 21.950 & 0.2847 & 5.645 & 0.351 & 0.748 \\
\bottomrule
\end{tabular}
}
\end{subtable}

\vspace{0.5em}

\begin{subtable}[t]{0.49\textwidth}
\centering
\caption{Effect of regularization strength $\lambda$.}
\label{tab:ablation_lambda}
\resizebox{\linewidth}{!}{
\begin{tabular}{c|ccccc}
\toprule
\textbf{Method} & \textbf{PickScore} & \textbf{HPS} & \textbf{Aesthetics} & \textbf{CLIP} & \textbf{ImageReward} \\
\midrule
$\lambda=0.00005$ & 21.323 & 0.2767 & 5.4990 & 0.3350 & 0.5700 \\
$\lambda=0.00025$ & \textbf{21.992} & \textbf{0.2860} & \textbf{5.6710} & \textbf{0.3485} & \textbf{0.8107} \\
$\lambda=0.00125$ & 21.733 & 0.2809 & 5.5287 & 0.3414 & 0.5343 \\
\bottomrule
\end{tabular}
}
\end{subtable}

\end{table}

\subsection{Robustness to Reward Noise}
\label{sec:reward_noise}

We evaluate LAIR's robustness to imperfect reward supervision by perturbing the training rewards with additive Gaussian noise,
$
\tilde{r}_i = r_i + \sigma \epsilon_i,
$
where $\epsilon_i \sim \mathcal{N}(0,1)$. We train SD1.5 models with $\sigma \in \{0.5, 1.0\}$ while keeping all other hyperparameters fixed, and evaluate on the full Parti-Prompts benchmark.

\begin{table}[h]
    \centering
    \caption{
        Robustness of LAIR-SD1.5 to additive Gaussian reward noise.
        Results are reported on Parti-Prompts with five generations per prompt.
    }
    \label{tab:reward_noise}
    \begin{tabular}{c|ccccc}
        \toprule
        \textbf{Reward noise $\sigma$} &
        \textbf{PickScore} &
        \textbf{HPS} &
        \textbf{Aesthetics} &
        \textbf{CLIP} &
        \textbf{ImageReward} \\
        \midrule
        $\sigma=0$   & \textbf{21.992} & \textbf{0.2860} & \textbf{5.6710} & 0.3485 & \textbf{0.8107} \\
        $\sigma=0.5$ & 21.871 & 0.2846 & 5.5630 & \textbf{0.3490} & 0.7639 \\
        $\sigma=1.0$ & 21.759 & 0.2834 & 5.5310 & 0.3484 & 0.6993 \\
        \bottomrule
    \end{tabular}
\end{table}

As shown in Table~\ref{tab:reward_noise}, LAIR remains robust to substantial corruption of the training rewards. Increasing the noise level from $\sigma=0$ to $\sigma=0.5$ causes only modest degradation across most metrics, with CLIP even improving slightly, while the model retains most of its gains over the base SD1.5 model. Even under the more severe setting of $\sigma=1.0$, LAIR continues to achieve strong preference and alignment scores rather than collapsing to baseline performance. The smooth, monotonic degradation suggests that LAIR does not rely on perfectly calibrated reward values and can tolerate considerable additive noise, while still benefiting from cleaner supervision when available.

\section{Additional Results} \label{sec:additional-results}

\begin{figure}[!h]
    \centering
    \vspace{-0.08in}
    \includegraphics[
        width=0.9\linewidth,
        height=0.9\textheight,
        keepaspectratio
    ]{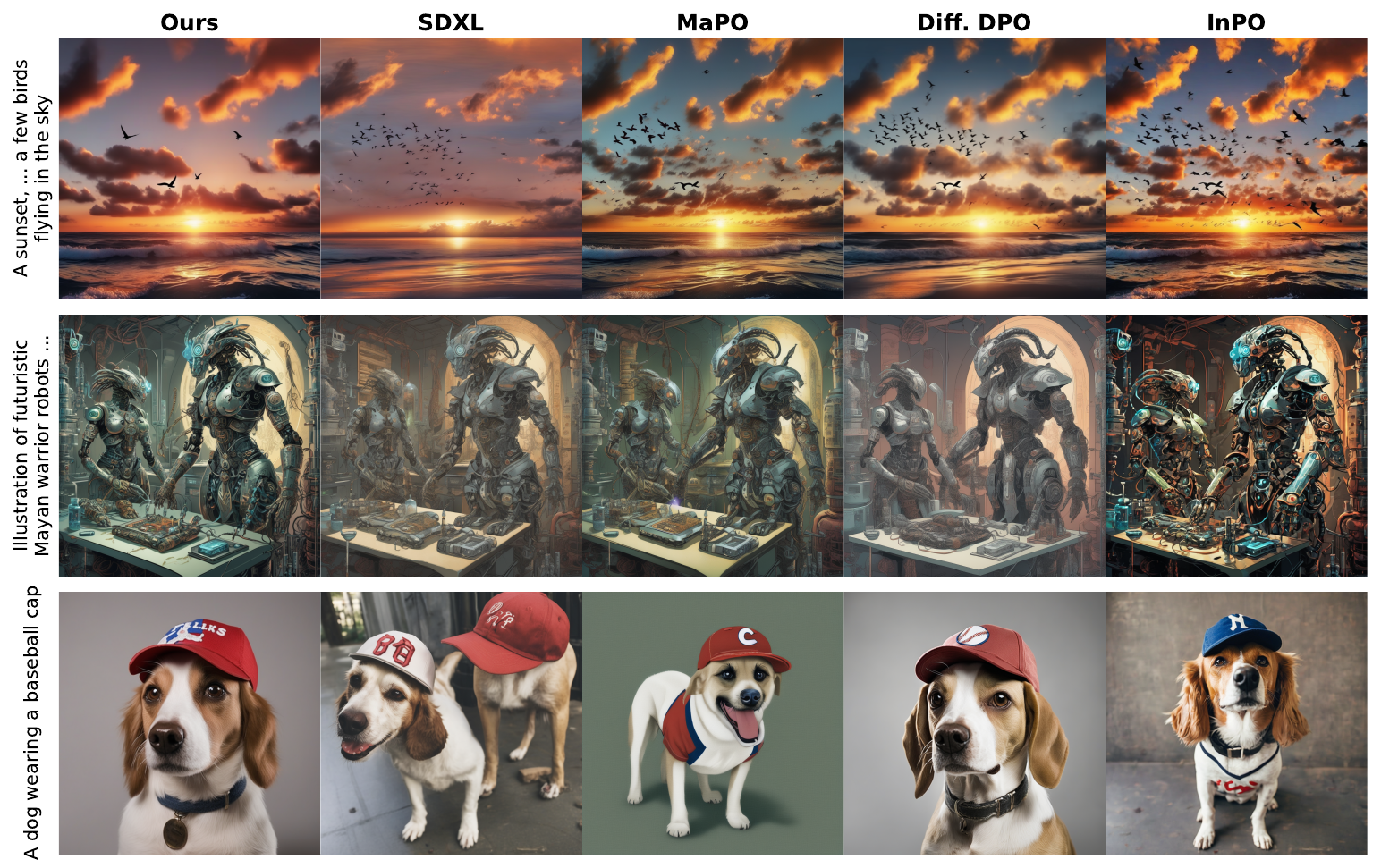}
    \caption{Images generated by Diffusion LAIR (Ours), SDXL, MaPO, Diffusion DPO, and InPO.}
    \label{fig:sdxl-compare-baselines}
    \vspace{-0.08in}
\end{figure}

GenEval images generated using our SD1.5-tuned model are provided in Figure \ref{fig:ours-geneval}, using the following prompts, ordered top to bottom and left to right: 

\begin{itemize}
    \item ``a photo of a green cup and a red pizza''
    \item ``a photo of a yellow handbag and a blue refrigerator''
    \item ``a photo of a green suitcase and a blue boat''
    \item ``a photo of a red cake''
    \item ``a photo of a white dog''
    \item ``a photo of an orange laptop''
    \item ``a photo of a green couch''
    \item ``a photo of three refrigerators''
    \item ``a photo of two backpacks''
    \item ``a photo of a potted plant and a backpack''
    \item ``a photo of a fork and a knife''
    \item ``a photo of two vases''
\end{itemize}

\begin{figure}[!h]
    \centering
    \includegraphics[
        width=1.0\linewidth,
        height=1.0\textheight,
        keepaspectratio
    ]{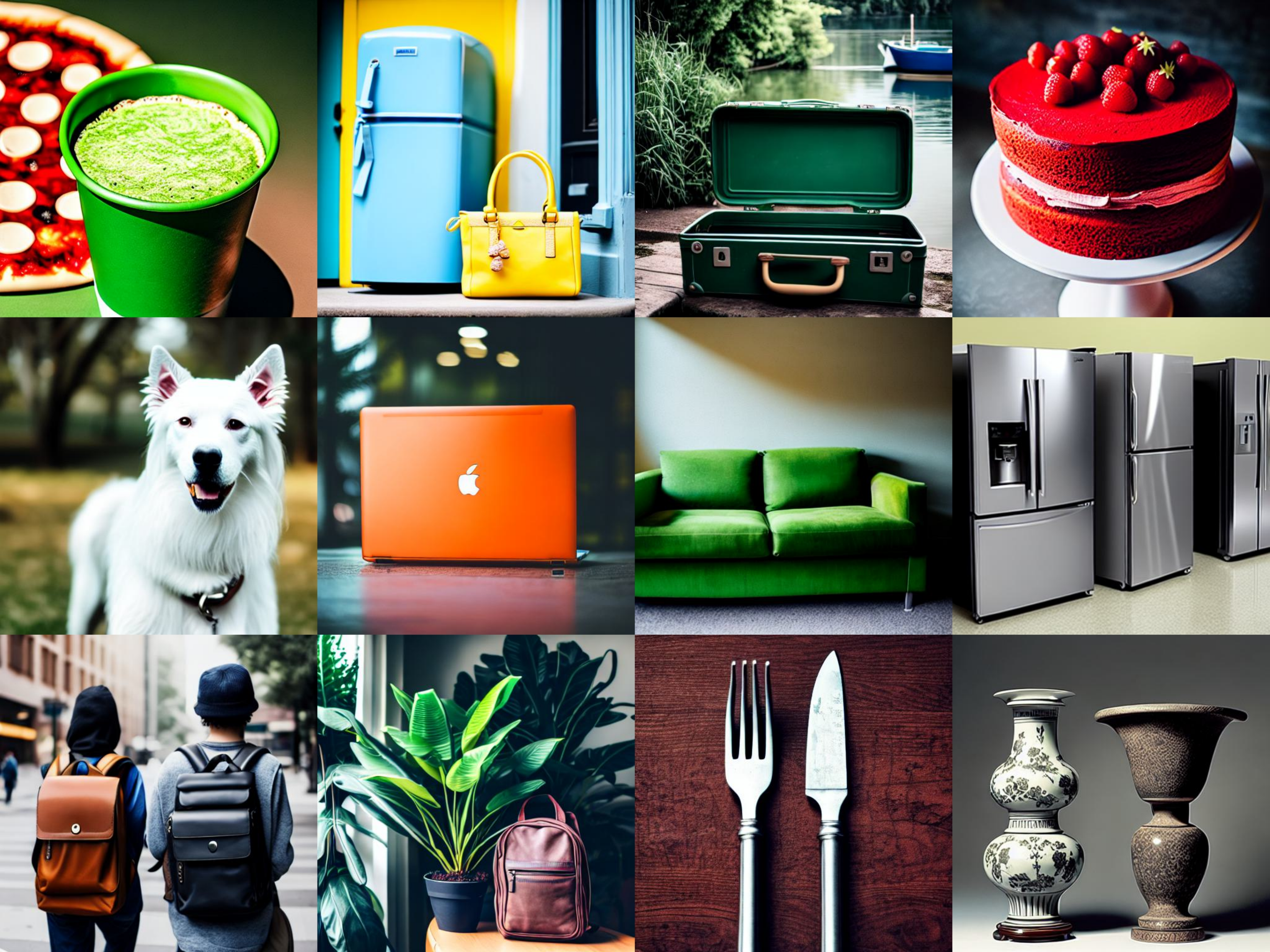}
    \caption{GenEval qualitative results generated by our SD1.5-tuned model.}
    \label{fig:ours-geneval}
\end{figure}

\begin{figure}[!h]
    \centering
    \includegraphics[
        width=1.0\linewidth,
        height=1.0\textheight,
        keepaspectratio
    ]{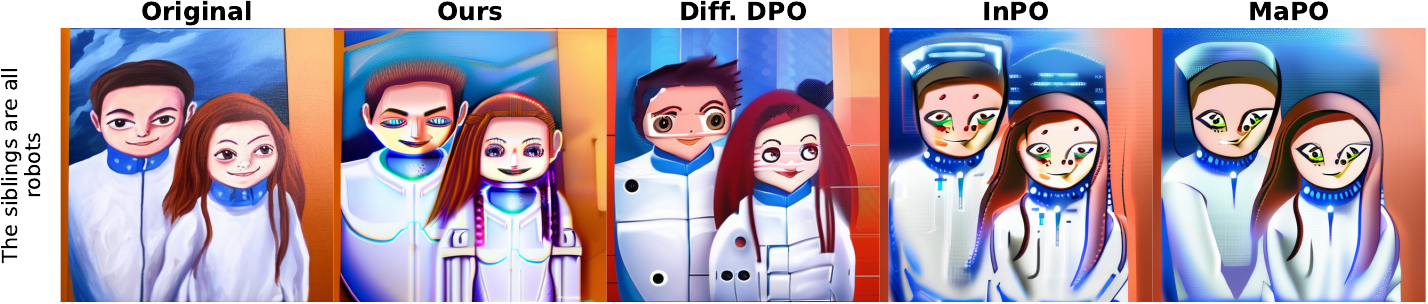}
    \caption{SDEdit qualitative results generated by the SDXL variants.}
    \label{fig:edit-compare}
\end{figure}


\section{Discussion} \label{sec:discussion}

\paragraph{Limitations} Our method relies on a reward model's scores during training. If the reward model itself is noisy or not well aligned with human preferences, the resulting fine-tuned model from our objective may not be well-aligned. The theoretical analysis, specifically for the surrogate KL bound in Section \ref{sec:theory}, relies on several assumptions: the approximation of the log-ratio via the ELBO, and the full-support extension of the implicit reward. We note that the log-ratio approximation is a standard assumption utilized by works such as \cite{wallace2024diffusion, xue2025advantage, bai2025towards}. Despite these assumptions, we still believe that the surrogate KL bound provides useful intuition about the behavior of our method, even if it is not an exact theoretical guarantee.

\paragraph{Broader Impacts} As demonstrated in our paper, post-training diffusion models with our objective leads to significant gains in alignment with human preferences. As our objective learns directly from groups of samples scored with rewards, we believe it is more robust to noisy preference labels in common datasets that arise from human annotation error. Our approach can serve as a building block towards high-quality image generation that is aligned with human preference.



\end{document}